\documentclass[10pt,twocolumn,letterpaper]{article}

\usepackage{cvpr}              

\definecolor{cvprblue}{rgb}{0.21,0.49,0.74}
\usepackage[pagebackref,breaklinks,colorlinks,allcolors=cvprblue]{hyperref}

\usepackage[utf8]{inputenc} 
\usepackage[T1]{fontenc}    
\usepackage{url}            
\usepackage{booktabs}       
\usepackage{amsfonts}       
\usepackage{nicefrac}       
\usepackage{microtype}      
\usepackage{xcolor}         

\usepackage{bbm}
\usepackage{bm}

\usepackage{wrapfig}

\usepackage{times}
\usepackage{graphicx} 

\usepackage{amsmath}
\usepackage{amssymb}

\usepackage{multirow}

\definecolor{RoyalBlue}{RGB}{65,105,225}
\definecolor{RedOrange}{RGB}{255,69,0}

\usepackage{makecell}
\usepackage{color, colortbl}
\definecolor{Gray}{gray}{0.9}
\definecolor{demphcolor}{RGB}{144,144,144}
\definecolor{airforceblue}{rgb}{0.36, 0.54, 0.66}

\usepackage{subcaption}

\usepackage{microtype}
\usepackage{graphicx}
\usepackage{booktabs} 
\usepackage{hyperref} 

\usepackage{amsmath}
\usepackage{amssymb}
\usepackage{mathtools}
\usepackage{amsthm}

\usepackage{algorithm}
\usepackage{algpseudocode}

\usepackage[capitalize,noabbrev]{cleveref}

\theoremstyle{plain}

\newtheorem{theorem}{Theorem}[section]
\newtheorem{proposition}[theorem]{Proposition}

\theoremstyle{definition}
\newtheorem{definition}[theorem]{Definition}
\newtheorem{assumption}[theorem]{Assumption}
\theoremstyle{remark}
\newtheorem{remark}[theorem]{Remark}

\newcommand{\bx}{\mathbf{x}}
\newcommand{\by}{\mathbf{y}}
\newcommand{\btheta}{\boldsymbol{\theta}}
\newcommand{\loss}{l}          
\newcommand{\batchloss}{L}      
\newcommand{\score}{s}          
\newcommand{\batch}{\mathcal{B}} 
\newcommand{\dataset}{\mathcal{D}} 
\newcommand{\reals}{\mathbb{R}}   
\newcommand{\naturals}{\mathbb{N}} 

\usepackage{enumitem}
\hypersetup{
    colorlinks=true,    
    linkcolor=blue,     
    urlcolor=blue,      
    hidelinks          
}

\usepackage{color, colortbl}
\definecolor{Gray}{gray}{0.9}
\definecolor{demphcolor}{RGB}{144,144,144}
\definecolor{airforceblue}{rgb}{0.36, 0.54, 0.66}

\usepackage{caption}
\usepackage{subcaption}

\usepackage{makecell} 

\def\paperID{} 
\def\confName{CVPR}
\def\confYear{2026}

\title{Batch Loss Score for Dynamic Data Pruning}

\author{
Qing Zhou$^{1}$,
Bingxuan Zhao$^{1}$,
Tao Yang$^{1}$,
Hongyuan Zhang$^{2}$,
Junyu Gao$^{1}$,
Qi Wang$^{1}$\thanks{Corresponding author.}
\\
 $^{1}$Northwestern Polytechnical University \\
 $^{2}$The University of Hong Kong
\\
\small{\texttt{\{chautsing, bxuanzhao202, ytao9464, hyzhang98, gjy3035, crabwq\}@gmail.com}}}

\begin{document}

\maketitle

    \begin{abstract}
        Dynamic data pruning accelerates deep learning by selectively omitting less informative samples during training. While per-sample loss is a common importance metric, obtaining it can be challenging or infeasible for complex models or loss functions, often requiring significant implementation effort. This work proposes the Batch Loss Score (BLS), a computationally efficient alternative using an Exponential Moving Average (EMA) of readily available batch losses to assign scores to individual samples. We frame the batch loss, from the perspective of a single sample, as a noisy measurement of its scaled individual loss, with noise originating from stochastic batch composition. It is formally shown that the EMA mechanism functions as a first-order low-pass filter, attenuating high-frequency batch composition noise. This yields a score approximating the smoothed and persistent contribution of the individual sample to the loss, providing a theoretical grounding for BLS as a proxy for sample importance. BLS demonstrates remarkable code integration simplicity (\textbf{three-line injection}) and readily adapts existing per-sample loss-based methods (\textbf{one-line proxy}). Its effectiveness is demonstrated by enhancing two such methods to losslessly prune \textbf{20\%-50\%} of samples across \textit{14 datasets}, \textit{11 tasks} and \textit{18 models}, highlighting its utility and broad applicability, especially for complex scenarios where per-sample loss is difficult to access. Code is available at \url{https://github.com/mrazhou/BLS}.
        \end{abstract}

\section{Introduction}
\label{sec:introduction}

The training of state-of-the-art deep neural networks often involves massive datasets and extensive computation, posing significant resource challenges \citep{radford2019language, dosovitskiy2020image,VPN, PiNI, ImageNet,MSCOCO1}. Dynamic data pruning techniques \citep{InfoBatch, seta} aim to mitigate this burden by improving training efficiency. These methods \citep{Coleman2020Selection, mindermann2022prioritized, ucb, he2024large} dynamically identify and selectively skip samples presumed less informative at different training stages, thereby focusing computational resources.

\begin{figure}
    \centering
    \includegraphics[width=0.5\textwidth]{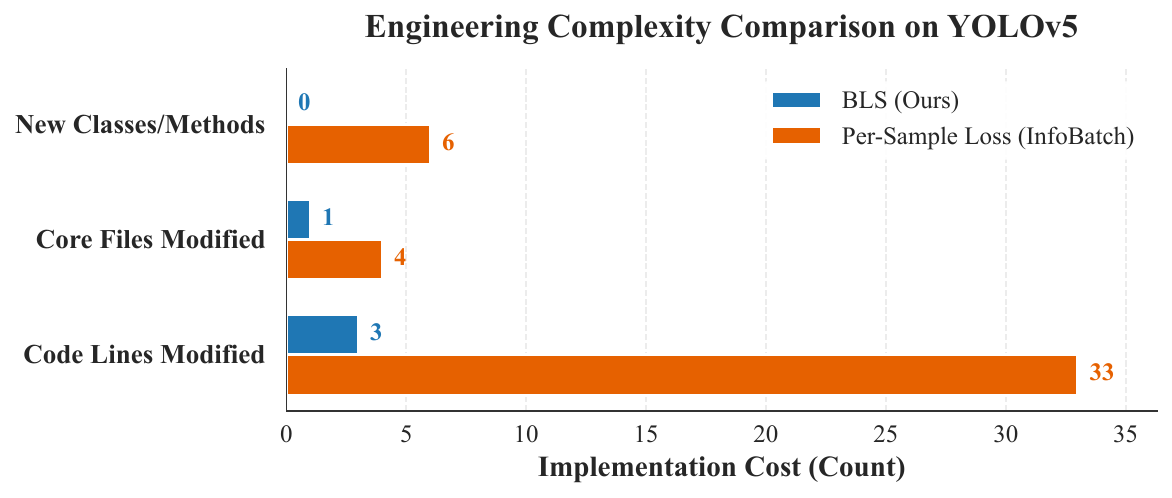}
    \caption{Score acquisition complexity (excluding scheduling logic). BLS: 3 line | InfoBatch: 33+ lines + intrusive modifications.}
    \vspace{-10pt}
    \label{fig:head_cost}
\end{figure}

A dominant strategy leverages per-sample loss $\loss_i(t)$ to gauge sample importance \citep{loshchilov2015online, paul2021deep,he2024large,ucb,InfoBatch,seta}, driven by the intuition linking higher loss to more informative samples. \textbf{However, directly accessing $\loss_i(t)$ presents significant practical barriers}. Standard deep learning pipelines are highly optimized around computing the \textit{mean batch loss}, $\batchloss(\batch_t, t)$, for gradient updates, making the recovery of individual $\loss_i(t)$ values post-aggregation non-trivial. This challenge is exacerbated by modern frameworks \citep{paszke2019pytorch} that abstract loss calculations and by complex objectives (e.g., multi-component losses in object detection \citep{yolo}) where defining and disentangling a representative per-sample scalar from the aggregate requires deep, task-specific knowledge and modifications \citep{yolo, usb, carion2020end, long2015fully, chen2017deeplab}. Thus, while $\batchloss(\batch_t, t)$ is ubiquitously available, the per-sample $\loss_i(t)$ remains largely inaccessible in many practical scenarios. This fundamental inaccessibility—spanning both engineering friction and conceptual difficulties—forms a critical bottleneck, significantly hindering the development and deployment of powerful loss-aware dynamic training strategies. Consequently, developing reliable methods to infer sample importance \textit{without} direct per-sample loss dependence is vital for advancing more adaptable and resource-efficient deep learning.

To circumvent these challenges, we introduce the Batch Loss Score (BLS), a novel mechanism \textbf{\textit{inferring sample importance solely from the readily available mean batch loss}}. BLS leverages the unique temporal sequence of batch losses encountered by each sample. It maintains per-sample scores, updated via an Exponential Moving Average (EMA) only when a sample is present in the current batch, incorporating that batch's mean loss. This sample-specific, conditional EMA update allows each score to reflect a unique participation history. As we demonstrate theoretically, this process inherently filters high-frequency noise stemming primarily from stochastic batch composition. The resulting smoothed score thus serves as a robust proxy for a sample's persistent loss contribution, achieving this with negligible computational overhead and exceptional implementation simplicity \textit{(requiring only 3 lines of code, detailed in supplementary material \cref{appendix:difficulty_per_sample_loss,sec:bls_simplicity_integration})}, facilitating easy integration into existing workflows. As shown in Figure~\ref{fig:head_cost}, BLS is significantly more simpler than existing methods, requiring only 3 lines of code to implement, while InfoBatch requires 33+ lines of code and intrusive modifications to the training pipeline.

The contributions of this paper are threefold.
\textbf{First}, we introduce BLS, a practical and efficient method for scoring sample importance using only mean batch losses.
\textbf{Second}, we provide a rigorous theoretical justification, analyzing BLS through a signal-processing lens and proving its noise-filtering properties under the frequency separation hypothesis.
\textbf{Third}, we conduct extensive empirical validation across diverse tasks and architectures, showing that BLS achieves losslessly training speedups when integrated into existing per-sample loss-based pruning methods, confirming its efficacy as a proxy, robust generalization, and utility without requiring per-sample loss access.
\section{Related Work}
\label{sec:related_work}
\textbf{Dynamic Data Selection.}
The idea of focusing computation on informative samples has a rich history, dating back to curriculum learning concepts \cite{bengio2009curriculum}, where models are initially trained on easier samples. Modern dynamic data selection methods automate this process throughout training. Many approaches \cite{katharopoulos2018not, mindermann2022prioritized, toneva2018empirical} aim to identify and prioritize hard or informative samples.
For instance, InfoBatch \cite{InfoBatch} implements soft pruning of samples exhibiting lower-than-average loss and corrects for the induced bias by rescaling gradients. SeTa \cite{seta} utilizes curriculum learning to select high-value samples in an easy-to-hard progression, achieving lossless acceleration during large-scale data training.
These frameworks often rely on specific importance metrics, such as per-sample loss, to guide the selection. BLS contributes to this area by proposing a novel scoring mechanism that can be readily integrated into existing dynamic selection frameworks that rely on sample scores.

\textbf{Sample Importance Metrics.}
Various metrics \cite{Coleman2020Selection,sener2017active,paul2021deep,iyer2021submodular} have been proposed to quantify sample importance for data selection. For dynamic data selection, per-sample loss is arguably the intuitive and frequently used metric \cite{loshchilov2015online, toneva2018empirical,InfoBatch,seta}, based on the premise that higher loss indicates greater difficulty or informativeness. Gradient norms \cite{katharopoulos2018not} or related quantities like gradient variance \cite{alain2015variance} provide another perspective, measuring the potential impact of a sample on parameter updates.
A primary challenge common to many of these metrics, particularly per-sample loss, is implementation complexity associated with obtaining them for every sample at frequent intervals during training. BLS differentiates itself by estimating sample importance indirectly through the temporal dynamics of the easily accessible batch loss, offering a significant advantage in terms of computational and implementation overhead.

\textbf{Leveraging Batch Statistics.}
While most selection methods focus on per-sample properties, some research utilizes batch-level statistics for other purposes, such as adaptive optimization (e.g., adaptive batch size methods \cite{smith2017don}). However, using temporal statistics of batch loss specifically to derive individual sample scores for data selection appears less explored. BLS is novel in its specific use of EMA-smoothed batch loss history as a direct proxy for individual sample importance over time. It exploits the readily available batch loss signal in a new way, leveraging its temporal dynamics through filtering to overcome the lack of immediate per-sample information.
\section{Methodology: Batch Loss Score (BLS)}
\subsection{Preliminaries and Problem Setup}
\label{sec:preliminaries}

Let $\dataset = \{(\bx_i, \by_i)\}_{i=1}^N$ be the training dataset of size $N$, $f(\cdot; \btheta)$ denote the model parameterized by $\btheta \in \reals^d$, and $\loss(f(\bx; \btheta), \by)$ be the per-sample loss function. At training iteration $t$, the model parameters are $\btheta_t$. The instantaneous loss for sample $i$ is $\loss_i(t) \coloneqq \loss(f(\bx_i; \btheta_t), \by_i)$.

Training typically involves sampling mini-batches $\batch_t \subset \dataset$ of size $B = |\batch_t|$ at each iteration $t$. The mean batch loss, used for gradient computation, is:
\begin{equation}
    \batchloss(\batch_t, t) = \frac{1}{B} \sum_{j \in \batch_t} \loss_j(t)
    \label{eq:batch_loss_def_final}
\end{equation}
Model parameters $\btheta_t$ are updated based on the gradient of $\batchloss(\batch_t, t)$ or a related objective, e.g., $\btheta_{t+1} = \btheta_t - \eta \nabla_{\btheta} \batchloss(\batch_t, t)|_{\btheta = \btheta_t}$.

\begin{definition}[Dynamic Data Pruning]
Dynamic data pruning is a process where, at certain steps $t$, an importance score $\score_i(t)$ is assigned to each sample $i \in \dataset$. Based on these scores, a subset $\dataset_t' \subseteq \dataset$ is selected for next training cycle (e.g., by pruning samples with $\score_i(t)$ below a threshold~\cite{InfoBatch}, or selecting samples within a specific difficulty window~\cite{seta}). Samples in $\dataset \setminus \dataset_t'$ are temporarily pruned for the next cycle.
\end{definition}

Our objective is to justify using a score derived solely from the sequence of observed $\batchloss(\batch_t, t)$ values as a meaningful proxy for the unobserved (or expensive) $\loss_i(t)$ for data pruning decisions.

\subsection{Per-Sample Score Estimation}
\label{sec:bls_method}

To overcome the lack of sample differentiation in the raw $\batchloss(\batch_t, t)$ and capture individual sample characteristics over time, we introduce the Batch Loss Score (BLS). Each sample $i \in \dataset$ is associated with a score $\score_i(t)$, initialized (e.g., $\score_i(0) = \batchloss(\batch_0, 0)$). At each training step $t$, after computing $\batchloss(\batch_t, t)$, the scores are updated based on batch participation:

\begin{definition}[Batch Loss Score]
The BLS score $\score_i(t)$ for sample $i$ at step $t$ is updated via EMA:
\begin{equation}
    \score_i(t) = \begin{cases} \alpha \score_i(t-1) + (1-\alpha) \batchloss(\batch_t, t) & \text{if } i \in \batch_t \\ \score_i(t-1) & \text{if } i \notin \batch_t \end{cases}
    \label{eq:ema_update_final}
\end{equation}
where $\alpha \in (0, 1)$ is the EMA decay factor, controlling the memory of the score.
\end{definition}

The crucial aspect of \cref{eq:ema_update_final} is that $\score_i(t)$ only incorporates information from batches $\batch_\tau$ where $i \in \batch_\tau$. This allows $\score_i(t)$ to accumulate a unique temporal history for sample $i$, distinct from other samples $j \neq i$.

\subsection{Seamless Proxy Integration}
BLS functions as a transparent proxy for existing dynamic pruning frameworks (e.g., InfoBatch, SeTa). Instead of intrusive per-sample loss extraction, BLS calculates $s_i(t)$ via the mean batch loss and feeds it directly into the framework's sampler as a drop-in replacement for the exact loss $l_i(t)$. Consequently, the downstream pruning algorithm remains completely agnostic to the score's origin, requiring zero modifications to its core scheduling logic or hyperparameters.

\section{Theoretical Analysis: BLS as Noise-Filtered Smoothed Loss}
\label{sec:analysis_final}

We first decomposes the mean batch loss $\batchloss(\batch_t, t)$ into a sample's scaled loss (signal) and batch composition noise (\cref{sec:loss_decomposition}). Then show BLS's Exponential Moving Average acts as a low-pass filter, attenuating this noise to estimate the sample's persistent loss contribution (\cref{sec:ema_filtering}).

\subsection{Decomposition of Batch Loss}
\label{sec:loss_decomposition}
Consider a specific sample $i$ and a batch $\batch_t$ containing it ($i \in \batch_t$). We can decompose the mean batch loss to isolate the direct contribution of $\loss_i(t)$:
\begin{equation}
    \batchloss(\batch_t, t) = \underbrace{\frac{1}{B} \loss_i(t)}_{\text{Scaled Signal } (\mathcal{S}_i)} + \underbrace{\frac{1}{B} \sum_{j \in \batch_t, j \neq i} \loss_j(t)}_{\text{Batch Composition Noise } (\mathcal{N}_i)}
    \label{eq:loss_decomp_final}
\end{equation}
This decomposition is central to our analysis. It separates the component directly related to sample $i$'s loss (the \textit{signal}, albeit scaled) from the variability induced by the other samples randomly co-selected in the batch (the \textit{noise}).

\begin{definition}[Batch Composition Noise]
For $i \in \batch_t$, the batch composition noise $N_i(\batch_t, t)$ relative to sample $i$ is:
\begin{equation}
    N_i(\batch_t, t) \coloneqq \frac{1}{B} \sum_{j \in \batch_t, j \neq i} \loss_j(t).
\end{equation}
This term represents the average loss of the other $B-1$ samples accompanying $i$ in batch $\batch_t$.
\end{definition}
In particular, $N_i(\batch_t, t)$ varies stochastically across different steps $t$ where $i$ is selected, primarily due to the random sampling of the set $\batch_t \setminus \{i\}$. The sequence of these noise terms for sample $i$ constitutes a major source of high-frequency fluctuation that the EMA aims to filter.

\subsection{EMA as a Filtering Mechanism}
\label{sec:ema_filtering}
Let $i \in \dataset$, $\{t_k\}_{k \in \naturals}$ denote the sequence of training steps where $i$ is included in the sampled batch ($i \in \batch_{t_k}$), and $\mathcal{L}_i = \{\batchloss(\batch_{t_k}, t_k)\}_{k=1}^\infty$ be the sequence of batch losses observed by sample $i$. Based on \cref{eq:loss_decomp_final}, we define the corresponding sequences for the scaled signal and noise:
$\mathcal{S}_i = \{\frac{1}{B}\loss_i(t_k)\}_{k=1}^\infty$ and $\mathcal{N}_i = \{N_i(\batch_{t_k}, t_k)\}_{k=1}^\infty$. By definition, $\mathcal{L}_i[k] = \mathcal{S}_i[k] + \mathcal{N}_i[k]$ for all $k \ge 1$.

The BLS update rule (\cref{eq:ema_update_final}) defines a discrete-time filtering operation on the subsequence of steps $\{t_k\}$. Let $s_i^{(k)} \coloneqq s_i(t_k)$ be the score of sample $i$ immediately after its $k$-th update (at step $t_k$). Ignoring the precise global timing $t_k$ and focusing on the update index $k$ for sample $i$, then:
\begin{equation}
    s_i^{(k)} = \alpha s_i^{(k-1)} + (1-\alpha) \mathcal{L}_i[k], \quad k \ge 1
    \label{eq:ema_indexed_final}
\end{equation}
with initial condition $s_i^{(0)} = s_i(t_0)$ (the score before the first update at $t_1$). This is the standard form of a first-order infinite impulse response (IIR) filter, representing a linear time-invariant (LTI) system \cite{oppenheim1999discrete} operating on the observation sequence $\mathcal{L}_i$.

\begin{assumption}[Frequency Separation Hypothesis]
\label{ass:frequency_final}
The stochastic fluctuations in the batch composition noise sequence $\mathcal{N}_i$, primarily driven by the random sampling of batch members, possess significant spectral energy at higher frequencies compared to the typical evolution of the scaled per-sample loss sequence $\mathcal{S}_i$. The latter reflects the comparatively slower changes in the model's performance on sample $i$ due to gradual parameter updates $\btheta_t$.
\end{assumption}

\begin{remark}
This assumption is grounded in the differing timescales of batch sampling (per-step randomness) and model learning (changes over multiple updates or epochs). While $\loss_i(t)$ can change, its persistent trend is usually smoother than the noise induced by constantly varying batch compositions.
\end{remark}

\begin{proposition}[BLS Score as Low-Pass Filtered Estimate]
\label{prop:filtered_estimate_final}
The BLS score $s_i^{(k)}$ derived from \cref{eq:ema_indexed_final} is the output of applying a first-order IIR low-pass filter $H_{\alpha}$ to the input sequence $\mathcal{L}_i = \mathcal{S}_i + \mathcal{N}_i$, plus a term decaying with the initial condition $s_i^{(0)}$:
\begin{equation}
    s_i^{(k)} = (H_{\alpha} * \mathcal{L}_i)[k] + \alpha^k s_i^{(0)}
    \label{eq:prop_conv_form_final}
\end{equation}
where $*$ denotes discrete convolution and $H_{\alpha}$ is the filter with impulse response $h[n] = (1-\alpha)\alpha^n u[n]$ ($u[n]$ being the Heaviside step function). By filter linearity, this decomposes into filtered signal and noise components:
\begin{equation}
    s_i^{(k)} = \underbrace{(H_{\alpha} * \mathcal{S}_i)[k]}_{\text{Smoothed Signal Trend}} + \underbrace{(H_{\alpha} * \mathcal{N}_i)[k]}_{\text{Smoothed Noise Baseline}} + \alpha^k s_i^{(0)}.
    \label{eq:prop_decomp_form_final}
\end{equation}
Invoking Assumption \ref{ass:frequency_final}, the low-pass characteristic of $H_{\alpha}$ ensures that high-frequency components, predominantly present in the noise $\mathcal{N}_i$, are attenuated in $(H_{\alpha} * \mathcal{N}_i)[k]$. Consequently, $s_i^{(k)}$ provides an estimate emphasizing the smoothed, persistent trend of the scaled per-sample loss $\mathcal{S}_i$, with reduced influence from stochastic batch variations.
\end{proposition}

\begin{proof}
    The update rule \cref{eq:ema_indexed_final}, $s_i^{(k)} = \alpha s_i^{(k-1)} + (1-\alpha) \mathcal{L}_i[k]$, defines the output $s_i^{(k)}$ of a discrete-time LTI system \cite{oppenheim1999discrete} with input sequence $\mathcal{L}_i[k] = \batchloss(\batch_{t_k}, t_k)$.
    The impulse response of this first-order recursive system is $h[n] = (1-\alpha)\alpha^n u[n]$, where $u[n]$ is the Heaviside step function.
    
    For a general input $\mathcal{L}_i$, the output of the LTI system is given by the convolution of the input with the impulse response, plus the zero-input response due to the initial condition:
    \begin{equation}
        \begin{split}
            s_i^{(k)} =& \sum_{j=0}^{k} h[j] \mathcal{L}_i[k-j] + \alpha^k s_i^{(0)} \\
            =& (H_\alpha * \mathcal{L}_i)[k] + \alpha^k s_i^{(0)},
        \end{split}
    \end{equation}
    where $H_\alpha$ denotes the filter with impulse response $h[n]$ and $*$ denotes discrete convolution. This confirms \cref{eq:prop_conv_form_final}.
    The filter $H_\alpha$ is a standard first-order IIR low-pass filter. Its frequency response is $H(e^{j\omega}) = \frac{1-\alpha}{1-\alpha e^{-j\omega}}$. The magnitude $|H(e^{j\omega})| = \frac{1-\alpha}{\sqrt{1 - 2\alpha \cos(\omega) + \alpha^2}}$ is maximized at $\omega=0$ (DC) and decreases monotonically as frequency $|\omega|$ increases towards $\pi$. This low-pass characteristic means the filter attenuates high-frequency components while passing low-frequency components. The parameter $\alpha \in (0, 1)$ controls the cutoff frequency; $\alpha \to 1$ implies stronger smoothing (lower cutoff).
    
    By the linearity property of convolution (and the filter $H_\alpha$), applying the filter to the decomposed input $\mathcal{L}_i = \mathcal{S}_i + \mathcal{N}_i$ (from \cref{eq:loss_decomp_final}) yields:
    \begin{equation}
        \begin{split}
            (H_\alpha * \mathcal{L}_i)[k] =& (H_\alpha * (\mathcal{S}_i + \mathcal{N}_i))[k] \\
            =& (H_\alpha * \mathcal{S}_i)[k] + (H_\alpha * \mathcal{N}_i)[k].
        \end{split}
    \end{equation}
    Substituting this into the expression for $s_i^{(k)}$ gives the decomposition in \cref{eq:prop_decomp_form_final}.
    
    Invoking Assumption \ref{ass:frequency_final}, the batch composition noise sequence $\mathcal{N}_i$ is assumed to have significant energy at higher frequencies compared to the scaled per-sample loss sequence $\mathcal{S}_i$. The low-pass filter $H_\alpha$ therefore primarily:
    Passes and smooths the lower-frequency content of $\mathcal{S}_i$, resulting in $(H_{\alpha} * \mathcal{S}_i)[k]$ which reflects the persistent trend of the scaled sample loss.
    Attenuates the higher-frequency content dominant in $\mathcal{N}_i$, resulting in a smoothed noise term $(H_{\alpha} * \mathcal{N}_i)[k]$ with reduced variance.
    Consequently,
    the BLS score $s_i^{(k)}$ predominantly captures the smoothed signal trend, effectively filtering out much of the noise induced by random batch composition. This allows $s_i^{(k)}$ to serve as a robust estimate of the sample's persistent loss contribution over time.
    \end{proof}

\subsection{Discussion Effectiveness of BLS}
\label{sec:discussion}

The theoretical analysis, culminating in Proposition~\ref{prop:filtered_estimate_final}, offers several insights into BLS's efficacy as a proxy for sample importance:

\textbf{Sample-Specific Filtered History.} The conditional EMA update (\cref{eq:ema_update_final}) ensures each sample score, $\score_i(t)$, aggregates information solely from batch losses observed during its active participation. Applying the filter $H_\alpha$ to this unique sequence generates a sample-specific history, distinguishing $\score_i(t)$ from $\score_j(t)$ for $i \neq j$.

\textbf{Filtering Batch Composition Randomness.} BLS's core mechanism is noise filtering. The EMA acts as a low-pass filter $H_\alpha$ which, under Assumption~\ref{ass:frequency_final}, effectively attenuates high-frequency variations inherent in the batch composition noise $\mathcal{N}_i$. This prevents scores from being skewed by transient, unrepresentative losses of randomly co-occurring samples within a single batch.

\textbf{Approximating Persistent Loss Contribution.} The BLS score (\cref{eq:prop_decomp_form_final}) approximates the sum of the smoothed scaled signal $(H_{\alpha} * \mathcal{S}_i)[k]$ and smoothed noise $(H_{\alpha} * \mathcal{N}_i)[k]$. The term $(H_{\alpha} * \mathcal{S}_i)[k]$ captures the smoothed historical trend of the sample's scaled loss contribution $(1/B)\loss_i(t)$. If the smoothed noise $(H_{\alpha} * \mathcal{N}_i)[k]$ acts as a relatively stable baseline, score differences $\score_i(t) - \score_j(t)$ are primarily driven by differences in their smoothed signal components. Thus, BLS effectively ranks samples by their persistent, time-averaged contribution to batch loss.

\textbf{Role of EMA Parameter $\alpha$.} The decay factor $\alpha$ governs the filter's characteristics. A larger $\alpha$ (closer to 1) implies stronger smoothing (lower cutoff frequency), leading to greater noise attenuation but reduced responsiveness to true changes in sample importance $\loss_i(t)$. Conversely, smaller $\alpha$ values yield more responsive but noisier scores. Optimal $\alpha$ selection balances noise suppression with tracking sample importance dynamics.
\definecolor{RoyalBlue}{RGB}{65,105,225}
\definecolor{nicegreen}{RGB}{0,180,0}
\newcommand{\blue}[1]{$_{\color{RoyalBlue}\downarrow #1}$}
\newcommand{\mred}[1]{$_{\color{nicegreen}\uparrow #1}$}

\begin{table*}[htbp]
    \centering
    \caption{BLS effectively proxies per-sample loss methods (InfoBatch, SeTa) on large-scale datasets. Results show comparable or improved performance for BLS-variants across diverse vision-language and scene text recognition tasks. (a) ToCa dataset (zero-shot captioning, 3M samples). (b) MJ+ST dataset (scene text recognition, 15M samples). (c) SS1M dataset (cross-domain captioning, 3M samples). Pruned \% indicates data reduction.}
    \label{tab:combined_tables_all_data}

    \begin{minipage}[t]{0.4\textwidth}
        \centering
        \footnotesize
        \setlength{\tabcolsep}{2.5pt}
        \par\vspace{1ex}
        \parbox{\linewidth}{\centering\textbf{(a) ViECap on ToCa}}\par\vspace{1ex}
        \begin{tabular}{@{}cc|cccc|c@{}}
            \toprule[0.9pt]
            \multirow{2}{*}{Method} & \multirow{2}{*}{Pruned \%} & \multicolumn{4}{c|}{NoCaps Val (CIDEr)} & COCO \\
            ~ & ~ & In & Near & Out & Overall & Test \\
            \midrule[0.9pt]
            Full & - & 63.2 & 68.8 & 70.2 & 70.5 & 95.2 \\
            \hline
            InfoBatch & 23.6 & 63.4 & 68.2 & 70.4 & 70.2 & 94.4 \\
            SeTa & 31.7 & 63.4 & 69.7 & 71.4 & 71.5 & 95.3 \\
            \hline
            BLS-InfoBatch & 23.8 & 65.1 & 68.9 & 69.8 & 70.6 & 94.7 \\
            BLS-SeTa & 32.0 & 64.3 & 69.4 & 71.0 & 71.2 & 95.1 \\
            \bottomrule[0.9pt]
        \end{tabular}
    \end{minipage}
    \hfill 
    \begin{minipage}[t]{0.29\textwidth}
        \centering
        \footnotesize
        \setlength{\tabcolsep}{2.5pt}
        \par\vspace{1ex}
        \parbox{\linewidth}{\centering\textbf{(b) ABINet on MJ+ST}}\par\vspace{1ex}
        \begin{tabular}{@{}c|ccccc@{}} 
            \toprule[0.9pt]
            \multirow{2}{*}{Pruned \%} & IIIT5k & SVT & IC15 & SVTP & CUTE80 \\
            ~ & \multicolumn{5}{c}{Accuracy \%} \\ 
            \midrule[0.9pt]
            -    & 96.1 & 93.4  & 85.4 & 88.7 & 89.2 \\ 
            \hline
            26.6 & 96.3 & 93.6 & 85.4 & 88.8 & 89.4 \\ 
            28.1 & 96.2 & 93.4 & 85.6 & 89.3 & 89.2 \\ 
            \hline
            25.4 & 96.1 & 93.8 & 85.5 & 89.1 & 89.9 \\ 
            33.0 & 96.2 & 94.0 & 85.5 & 88.7 & 89.1 \\ 
            \bottomrule[0.9pt]
        \end{tabular}
    \end{minipage}
    \hfill 
    \begin{minipage}[t]{0.27\textwidth}
        \centering
        \footnotesize
        \setlength{\tabcolsep}{2.5pt}
        \par\vspace{1ex}
        \parbox{\linewidth}{\centering\textbf{(c) ViECap on SS1M}}\par\vspace{0.9ex}
        \begin{tabular}{@{}c|cc|cc@{}} 
            \toprule[0.9pt]
            \multirow{2}{*}{Pruned \%} & \multicolumn{2}{c|}{COCO} & \multicolumn{2}{c}{Flickr30k} \\
            ~ & B@4 & C & B@4 & C \\
            \midrule[0.9pt]
            -    & 9.6 & 45.1 & 6.5 & 22.3 \\ 
            \hline
            23.0 & 9.2  & 44.4 & 6.4 & 22.4 \\ 
            40.0 & 9.5 & 45.6 & 6.4 & 22.0 \\ 
            \hline
            24.9 & 9.3  & 44.5 & 6.5 & 22.5 \\ 
            41.1 & 9.6 & 45.5 & 6.5 & 22.2 \\ 
            \bottomrule[0.9pt]
        \end{tabular}
    \end{minipage}
\end{table*}

\begin{table}[h]
    \centering
    \caption{Comparison of state-of-the-art methods on CIFAR classification (ResNet18, Accuracy \%) at different pruning rates.
    BLS effectively proxies per-sample loss methods to achieve comparable accuracy to original methods under the same pruning rate.
    $^{\dag}$: adjusted epochs in original paper. $^{\ddagger}$: same epochs as other methods. \textbackslash{}: unattainable ratio. Random*: dynamic random.
    Performance relative to Full training is indicated ($\downarrow$: drop, $\uparrow$: gain).
    }
    \label{tab:sota}
    \footnotesize
    \setlength{\tabcolsep}{0pt}
    \begin{tabular}{c|ccc|ccc}
        \toprule[0.9pt]
        \multirow{2}{*}{Method} & \multicolumn{3}{c|}{CIFAR10} & \multicolumn{3}{c}{CIFAR100} \\
        & 30\% & 50\% & 70\% & 30\% & 50\% & 70\% \\
        \midrule[0.9pt]
        ResNet18 & \multicolumn{3}{c|}{95.6} & \multicolumn{3}{c}{78.2} \\
        \hline
        Random & 94.6 \blue{1.0} & 93.3 \blue{2.3} & 90.2 \blue{5.4} 
            & 73.8 \blue{4.4} & 72.1 \blue{6.1} & 69.7 \blue{8.5} \\

        GraNd-4~\citep{paul2021deep}
            & 95.3 \blue{0.3} & 94.6 \blue{1.0} & 91.2 \blue{4.4} & 74.6 \blue{3.6} & 71.4 \blue{6.8} & 68.8 \blue{9.4} \\
        EL2N-20~\citep{toneva2018empirical}
            & 95.3 \blue{0.3} & 95.1 \blue{0.5} & 91.9 \blue{3.7} & 77.2 \blue{1.0}& 72.1 \blue{6.1} & - \\
        DP~\citep{yang2023dataset}
            & 94.9 \blue{0.7} & 93.8 \blue{1.8} & 90.8 \blue{4.8} & 77.2 \blue{1.0} & 73.1 \blue{5.1} & - \\
            
        Random* & 94.8 \blue{0.8} & 94.5 \blue{1.1} & 93.0 \blue{2.6} 
            & 77.3 \blue{0.9} & 75.3 \blue{2.9} & - \\
        $\epsilon$-greedy~\cite{ucb} & 95.2 \blue{0.4} & 94.9 \blue{0.7} & 94.1 \blue{1.5}
            & 76.4 \blue{1.8}& 74.8 \blue{3.4} & - \\
        UCB~\cite{ucb} & 95.3 \blue{0.3} & 94.7 \blue{0.9} & 93.9 \blue{1.7} 
            & 77.3 \blue{0.9} & 75.3 \blue{2.9} & - \\
  
        InfoBatch$^{\dag}$~\cite{InfoBatch} & \colorbox{lightgray}{95.6 \mred{0.0}} & 95.1 \blue{0.5} &  94.7 \blue{0.9} 
            & \colorbox{lightgray}{78.2 \mred{0.0}}& 78.1 \blue{0.1} & 76.5 \blue{1.7} \\
        \hline
        InfoBatch$^{\ddagger}$~\cite{InfoBatch} & \colorbox{lightgray}{95.6 \mred{0.0}} & 95.0 \blue{0.6} & 94.4 \blue{1.2}
            & \colorbox{lightgray}{78.3 \mred{0.1}} & 77.7 \blue{0.5} & \textbackslash \\
        SeTa \cite{seta} & \colorbox{lightgray}{95.7 \mred{0.1}} & 95.3 \blue{0.3} & 95.0 \blue{0.6}
            & \colorbox{lightgray}{78.4 \mred{0.2}} & 78.0 \blue{0.2} & 76.7 \blue{1.6} \\
          
        \hline
        BLS-InfoBatch & \colorbox{lightgray}{95.6 \mred{0.0}} & 95.1 \blue{0.5} & 94.5 \blue{1.1}
            & \colorbox{lightgray}{78.4 \mred{0.2}} & 77.7 \blue{0.5} & \textbackslash \\
        BLS-SeTa & \colorbox{lightgray}{95.6 \mred{0.0}} & 95.3 \blue{0.3} & 95.0 \blue{0.6}
            & \colorbox{lightgray}{78.5 \mred{0.3}} & 78.1 \blue{0.1} & 76.7 \blue{1.5} \\
  
        \bottomrule[0.9pt]
    \end{tabular}
  \end{table}

\section{Experiments}
\label{sec:experiments}
We conduct extensive experiments, that include 14 datasets, 11 tasks, 18 model architectures and 2 pruning methods, to evaluate the \textit{effectiveness}, \textit{efficiency}, and \textit{generalization} capabilities of the proposed BLS.
Our evaluation strategy is designed to progressively demonstrate BLS's capabilities: first establishing its validity as an effective proxy for per-sample loss in standard settings, then showcasing its competitive performance and broad generalization across challenging tasks and architectures, and finally, empirically validating its underlying theoretical mechanism.

\textbf{Tasks and Datasets.} Evaluation spans classification (CIFAR \cite{cifar100}, ImageNet \cite{ImageNet}), object detection/segmentation (COCO \cite{MSCOCO1}), vision-language (image/video captioning on COCO \cite{COCO-cap}/NoCaps \cite{NoCaps}/SS1M \cite{SS1M}/ToCa \cite{ToCa}/MSR-VTT \cite{MSRVTT}), scene text recognition on MJ+ST (\cite{MJ,ST}), image generation (MNIST \cite{minst}, CIFAR10), multi-view stereo (WHU-MVS \cite{WHU-MVS}), and semi-supervised learning \cite{usb} (image classification on CIFAR100, text classification on Yelp Review \cite{yelp_dataset}, audio classification on ESC-50 \cite{esc50}).

\textbf{Architectures.} We test across CNNs (ResNet18/50 \citep{resnet}, EfficientNet \citep{tan2019efficientnet}, YOLOv5 \cite{yolo}), Transformers (ViT \citep{dosovitskiy2020image}, Swin \citep{liu2021swin}), and emerging architectures Vim \citep{vim}. Vision-language tasks use ViECap \citep{viecap} or ABINet \citep{ABINet} backbones. Generative models include VAE \cite{DBLP:journals/corr/KingmaW13} and DDPM \cite{ho2020denoising,ho2022classifier}. MVS uses Ada-MVS \citep{Ada-MVS}. Semi-supervised learning uses FixMatch \cite{sohn2020fixmatch}/FlexMatch \cite{zhang2021flexmatch}/Dash \cite{xu2021dash}.

\textbf{BLS Integration.} We evaluate BLS by integrating it into two representative per-sample loss-based pruning methods: InfoBatch \citep{InfoBatch} and SeTa \citep{seta}, denoted as BLS-InfoBatch and BLS-SeTa. 
This allows direct assessment of BLS as a drop-in replacement for $\loss_i(t)$ scoring; such integration is achieved with the minimal \textbf{\textit{one-line proxy}} and \textbf{\textit{three-line injection}} mechanisms detailed in Suppl. \cref{sec:bls_simplicity_integration}.
To ensure fairness, all experiments adopt the established settings corresponding to their respective tasks, models, and datasets.

\textbf{Evaluation Metrics.} Model performance is measured using standard task-specific metrics. Efficiency is primarily evaluated by the percentage of data pruned \textbf{(Pruned \%)} while maintaining comparable performance to the baseline or original method. We justify this choice in Suppl. \cref{appendix:efficiency_details}.

\begin{table*}[h]
    \caption{BLS demonstrates strong cross-architecture generalization on ImageNet-1K and CIFAR100 with CNNs, Transformers, and Mambas. Format: Accuracy / Pruned \%.}
    \label{tab:arch}
    \centering
    \footnotesize
    \setlength{\tabcolsep}{1.2pt}
    \begin{tabular}{c|ccc|cc|c|cc@{}}
        \toprule[0.9pt]
        \multirow{3}{*}{Method} & \multicolumn{6}{c|}{ImageNet-1K} & \multicolumn{2}{c}{\multirow{2}{*}{CIFAR100}} \\
        ~ & \multicolumn{3}{c}{CNN} & \multicolumn{2}{c}{Transformer} & Mamba & ~ & ~\\
        \cmidrule(lr){2-4} \cmidrule(lr){5-6}  \cmidrule(lr){7-7}
        ~ & R18 & R50 & EfficientNet & ViT & Swin & Vim & R18 & R50 \\
        \midrule[0.9pt]
        Full
        & 69.5 & 78.6 & 76.1
        & 73.3  & 80.0 & 75.7
        & 78.2 & 80.6 \\
        \hline

        \hline
        \multirow{1}{*}{BLS-InfoBatch}
            & 69.4 / 32.8 & 78.2 / 31.5 & 75.7 / 32.8
            & 73.2 / 26.2  & 80.0 / 30.8 & 75.6 / 22.9 
            & 78.3 / 33.9 & 80.5 / 38.6 \\
        \multirow{1}{*}{BLS-SeTa}
            & 69.5 / 31.3 & 78.4 / 32.3 & 76.0 / 31.1
            & 73.3 / 25.2  & 80.0 / 44.5 & 75.6 / 29.4 
            & 78.4 / 40.1 & 80.7 / 40.7 \\
        \bottomrule[0.9pt]
    \end{tabular}
\end{table*}

\begin{table*}[ht]
    \caption{Consolidated results for various methods on different datasets and tasks.
    (a) ViECap on COCO dataset: Image captioning performance with 556K samples.
    (b) ViECap on MSR-VTT dataset: Video captioning performance with 180K samples.
    (c) Ada-MVS on WHU-MVS dataset: Multi-view stereo performance with 28K image-depth map samples.}
    \label{tab:combined_tables_set2_horizontal}
    \noindent 
    \begin{minipage}[t]{0.55\textwidth}
        \centering
        \footnotesize 
        \setlength{\tabcolsep}{3pt} 

        \par\vspace{1ex}
        \parbox{\linewidth}{%
          \centering
          \textbf{(a) Image Captioning}
        }
        \par\vspace{1ex}

        \begin{tabular}{@{}cc|cccc|cccc@{}} 
            \toprule[0.9pt]
            \multirow{2}{*}{Method} & \multirow{2}{*}{Pruned \%} & \multicolumn{4}{c|}{NoCaps Val (CIDEr)} & \multicolumn{4}{c}{COCO} \\
            ~ & ~ & In & Near & Out & Overall & B@4 & M & C & S \\
            \midrule[0.9pt]
            Full & - 
                & 58.4 & 63.1 & 65.3 & 65.2
                & 27.1 & 24.6 & 91.5 & 18.0
                \\
            \hline
            BLS-InfoBatch & 20.9
                & 58.8 & 63.1 & 65.8 & 65.3
                & 27.1 & 24.6 & 91.9 & 17.9
                \\
            BLS-SeTa & 28.1
                & 58.9 & 63.1 & 65.6 & 65.3
                & 27.2 & 24.7 & 92.0 & 18.0
                \\
            \bottomrule[0.9pt]
        \end{tabular}
    \end{minipage}
    \hfill
    \begin{minipage}[t]{0.22\textwidth}
        \centering
        \footnotesize 
        \setlength{\tabcolsep}{3pt} 

        \par\vspace{1ex}
        \parbox{\linewidth}{%
          \centering
          \textbf{(b) Video Captioning}
        }
        \par\vspace{1ex}

        \begin{tabular}{@{}c|cccc@{}} 
            \toprule[0.9pt]
            \multirow{2}{*}{Pruned \%} & \multicolumn{4}{c}{MSR-VTT} \\
            ~ & B@4 & M & C & S \\
            \midrule[0.9pt]
            - & 23.4 & 20.8 & 27.9 & 5.0 \\
            \hline
            21.8 & 24.7 & 21.0 & 29.9 & 5.0 \\
            40.3 & 24.7 & 21.3 & 29.3 & 4.9 \\
            \bottomrule[0.9pt]
        \end{tabular}
    \end{minipage}
    \hfill
    \begin{minipage}[t]{0.16\textwidth}
        \centering
        \footnotesize 
        \setlength{\tabcolsep}{3pt} 

        \par\vspace{1ex}
        \parbox{\linewidth}{%
          \centering
          \textbf{(c) Multi-View Stereo}
        }
        \par\vspace{1ex}

        \begin{tabular}{@{}c|cc@{}} 
            \toprule[0.9pt]
            \multirow{2}{*}{Pruned \%} & \multicolumn{2}{c}{WHU-MVS} \\
            ~ & {<}3-i & {<}0.6m \\ 
            \midrule[0.9pt]
            - & 95.01 & 97.38 \\
            \hline
            45.0 & 95.17 & 97.33 \\
            37.5 & 95.04 & 97.45 \\
            \bottomrule[0.9pt]
        \end{tabular}
    \end{minipage}
\end{table*}

\subsection{Effective Proxy for Per-Sample Loss}
\label{sec:bls_as_proxy_and_sota_comparison} 


We first validate BLS's ability to serve as an effective proxy by integrating it into established per-sample loss-based pruning methods, InfoBatch \citep{InfoBatch} and SeTa, and comparing their performance with their BLS-enhanced counterparts (BLS-InfoBatch and BLS-SeTa).

\Cref{tab:combined_tables_all_data} presents these critical comparisons across three challenging large-scale vision-language and scene text recognition benchmarks. On the ToCa dataset for visual captioning (Table~\ref{tab:combined_tables_all_data}a), BLS-InfoBatch and BLS-SeTa achieve nearly identical or even slightly improved CIDEr scores (e.g., BLS-InfoBatch Overall CIDEr 70.6 vs. InfoBatch 70.2; BLS-SeTa 71.2 vs. SeTa 71.5) while maintaining comparable data pruning percentages (23.8\% and 30.0\% respectively). Similarly, for scene text recognition on the 15M-sample MJ+ST dataset (Table~\ref{tab:combined_tables_all_data}b), BLS-variants demonstrate strong parity. For instance, BLS-SeTa (pruning 33.0\%) achieves an accuracy of 94.0\% on SVT, matching or exceeding the original SeTa (93.4\% at 28.1\% pruning). The trend continues on the SS1M dataset for cross-domain image captioning (Table~\ref{tab:combined_tables_all_data}c), where BLS-SeTa (41.1\% pruned) yields a COCO CIDEr of 45.5, on par with SeTa (45.6 at 40.0\% pruned). These results across diverse, large-scale datasets with complex underlying losses strongly suggest that BLS effectively captures the essential sample importance signals that these sophisticated per-sample loss methods rely upon.

Further corroborating these findings, Table~\ref{tab:sota} details performance on standard dynamic data pruning benchmarks, CIFAR10 and CIFAR100, using ResNet18. When compared directly with InfoBatch and SeTa, the BLS-integrated versions consistently deliver statistically indistinguishable accuracies across various pruning rates (30\%, 50\%, 70\%).
Critically, both BLS-InfoBatch and BLS-SeTa often achieve performance on par with the fully trained baseline at moderate pruning rates. This not only confirms BLS's efficacy as a proxy but also hints at its potential to regularize training by focusing on more informative samples.

\begin{table}[ht]
    \centering
    \captionof{table}{Image generation with different architectures (VAE on MNIST, DDPM on CIFAR10, DDPM with Classifier Guidance (CFG) on CIFAR10). FID: Frechet Inception Distance.}
    \label{tab:vae_mnist}
    \footnotesize
    \setlength{\tabcolsep}{2pt}
    \begin{tabular}{c|cccccc} 
        \toprule[0.9pt]
        \multirow{2}{*}{Method} 
            & \multicolumn{2}{c}{VAE}
            & \multicolumn{2}{c}{DDPM}
            & \multicolumn{2}{c}{DDPM-CFG} \\
        \cmidrule(lr){2-3} \cmidrule(lr){4-5} \cmidrule(lr){6-7}
        ~ & Pruned \% & FID & Pruned \% & FID & Pruned \% & FID \\
        \midrule[0.9pt]
        Full & - & 35.34 
            & - & 16.38
            & - & 14.89 \\
        \hline
        BLS-InfoBatch & 28.9 & 35.55
            & 28.7 & 17.01
            & 22.0 & 14.99 \\
        BLS-SeTa & 34.0 & 35.34
            & 30.8 & 16.35
            & 22.3 & 14.72 \\
        \bottomrule[0.9pt]
    \end{tabular}
\end{table}

\begin{table}[ht]
    \centering
    \captionof{table}{Semi-supervised learning including image classification (FixMatch on CIFAR100), text sentiment classification (FlexMatch on Yelp Review) and audio classification (Dash on ESC-50).}
    \label{tab:semi_cifar100}
    \footnotesize
    \setlength{\tabcolsep}{2pt}
    \begin{tabular}{c|cc|cc|cc}
        \toprule[0.9pt]
        \multirow{2}{*}{Method} & \multicolumn{2}{c|}{Image} & \multicolumn{2}{c|}{Text} & \multicolumn{2}{c}{Audio} \\
        \cmidrule(lr){2-3} \cmidrule(lr){4-5} \cmidrule(lr){6-7}
        ~ & Pruned \% & Acc & Pruned \% & Acc & Pruned \% & Acc \\
        \midrule[0.9pt]
        Full & - & 61.9 & - & 53.7 & - & 64.5 \\
        \hline
        BLS-InfoBatch 
            & 31.3 & 61.9 
            & 38.1 & 54.4 
            & 25.0 & 64.5 \\
        BLS-SeTa 
            & 47.4 & 61.9 
            & 49.4 & 54.7 
            & 28.7 & 64.6 \\
        \bottomrule[0.9pt]
    \end{tabular}
\end{table}


\begin{table*}[h]
    \caption{Enhancing computational efficiency of a unified YOLOv5 framework for image classification on CIFAR100, alongside object detection and instance segmentation on COCO, through BLS.}
    \label{tab:yolo_coco}
    \centering
    \footnotesize
    \begin{tabular}{c|ccc|ccc|ccc}
        \toprule[0.9pt]
        \multirow{2}{*}{Method} & \multicolumn{3}{c|}{Image Classification} & \multicolumn{3}{c|}{Object Detection} & \multicolumn{3}{c}{Instance Segmentation} \\
        \cmidrule(lr){2-4} \cmidrule(lr){5-7} \cmidrule(lr){8-10}
        ~ & Pruned \% & Acc@1 & Acc@5 & Pruned \% & mAP$_{.5}$ & mAP$_{.5:.95}$ & Pruned \% & mAP$_{.5}$ & mAP$_{.5:.95}$ \\
        \midrule[0.9pt]
        YOLOv5n 
            & - & 58.0 & 84.6
            & - & 38.2 & 22.1 
            & - & 29.5 & 16.3
            \\
        \hline

        BLS-InfoBatch 
            & 20.4 & 57.6 & 84.0
            & 21.6 & 38.5 & 22.2 
            & 29.5 & 31.0 & 17.2
            \\
        BLS-SeTa 
            & 35.4 & 58.0 & 84.6
            & 25.2 & 38.3 & 22.1 
            & 36.1 & 31.7 & 17.6
            \\
        \bottomrule[0.9pt]
    \end{tabular}
\end{table*}

\subsection{Broad Generalization of Batch Loss Signal}
\label{sec:generalization_power}

A cornerstone of BLS's utility is its inherent generalizability, stemming from its sole reliance on the ubiquitous mean batch loss signal. This independence from specific loss or architectural details allows broad applicability, which we demonstrate across diverse models and tasks.

\textbf{Architectural Agnosticism.} BLS exhibits robust performance across varied architectures. Table~\ref{tab:arch} shows that for ImageNet-1K and CIFAR100, BLS-guided pruning maintains accuracy close to full training across CNNs (ResNets, EfficientNet), Transformers (ViT, Swin), and even Mamba (Vim). For instance, Vim with BLS-InfoBatch on ImageNet-1K prunes 22.9\% data while nearly matching full accuracy (75.6\% vs. 75.7\%). This efficacy across fundamentally different architectures underscores BLS's versatility, requiring no architecture-specific adaptations.

\textbf{Effectiveness in Complex Vision Tasks.} BLS navigates tasks with multi-component losses, where per-sample loss $\loss_i(t)$ is arduous to define or access (\cref{appendix:difficulty_per_sample_loss}). Table~\ref{tab:yolo_coco} (YOLOv5) shows BLS-variants maintain or improve mAP for object detection and instance segmentation while pruning significantly (e.g., BLS-SeTa +2.2 mAP$_{.5}$ for segmentation with 36.1\% pruning). Similarly, for Multi-View Stereo (Table~\ref{tab:combined_tables_set2_horizontal}c, WHU-MVS), BLS-SeTa prunes 37.5\% data with slight accuracy gains. These results highlight BLS's advantage in bypassing intricate loss disentanglement.

\textbf{Generalization to Vision-Language (VL) Scenarios.} BLS robustly extends to VL tasks involving complex cross-modal modeling and large-scale, potentially noisy datasets. Results on image/video captioning (Table~\ref{tab:combined_tables_all_data}, Table~\ref{tab:combined_tables_set2_horizontal}a-b) show BLS-variants maintaining or improving metrics like CIDEr and B@4 (e.g., COCO captioning CIDEr 92.0 with BLS-SeTa vs. 91.5 baseline, pruning 28.1\%) while pruning 20-40\% of data. This underscores BLS's utility without needing specialized per-sample loss definitions for intricate multimodal signals.

\textbf{Adaptability to Diverse Learning Paradigms.} Beyond supervised recognition and VL tasks, BLS demonstrates compatibility with fundamentally different learning approaches.
\textit{Generative Modeling (Table~\ref{tab:vae_mnist}):} For VAEs and Diffusion Models (DDPM/DDPM-CFG), BLS-guided pruning (22-34\%) preserves FID scores, indicating high-quality generation with reduced data.
\textit{Semi-Supervised Learning (SSL, Table~\ref{tab:semi_cifar100}):} BLS readily integrates into SSL for image, text, and audio classification, enabling substantial data pruning (25-49\%) without degrading accuracy. This is notable as SSL consistency losses often defy traditional hardness metric decomposition.

This extensive empirical validation across diverse tasks, architectures, complex losses, and learning paradigms robustly confirms that BLS's reliance on the universal mean batch loss grants it exceptional generalization. It stands as a highly practical and versatile tool for enhancing training efficiency across the machine learning landscape.

\subsection{Empirical Validation of Frequency Separation}
\label{sec:exp_freq_validation}

\begin{figure} 
    \centering 
    \includegraphics[width=0.5\textwidth, height=0.6\textwidth]{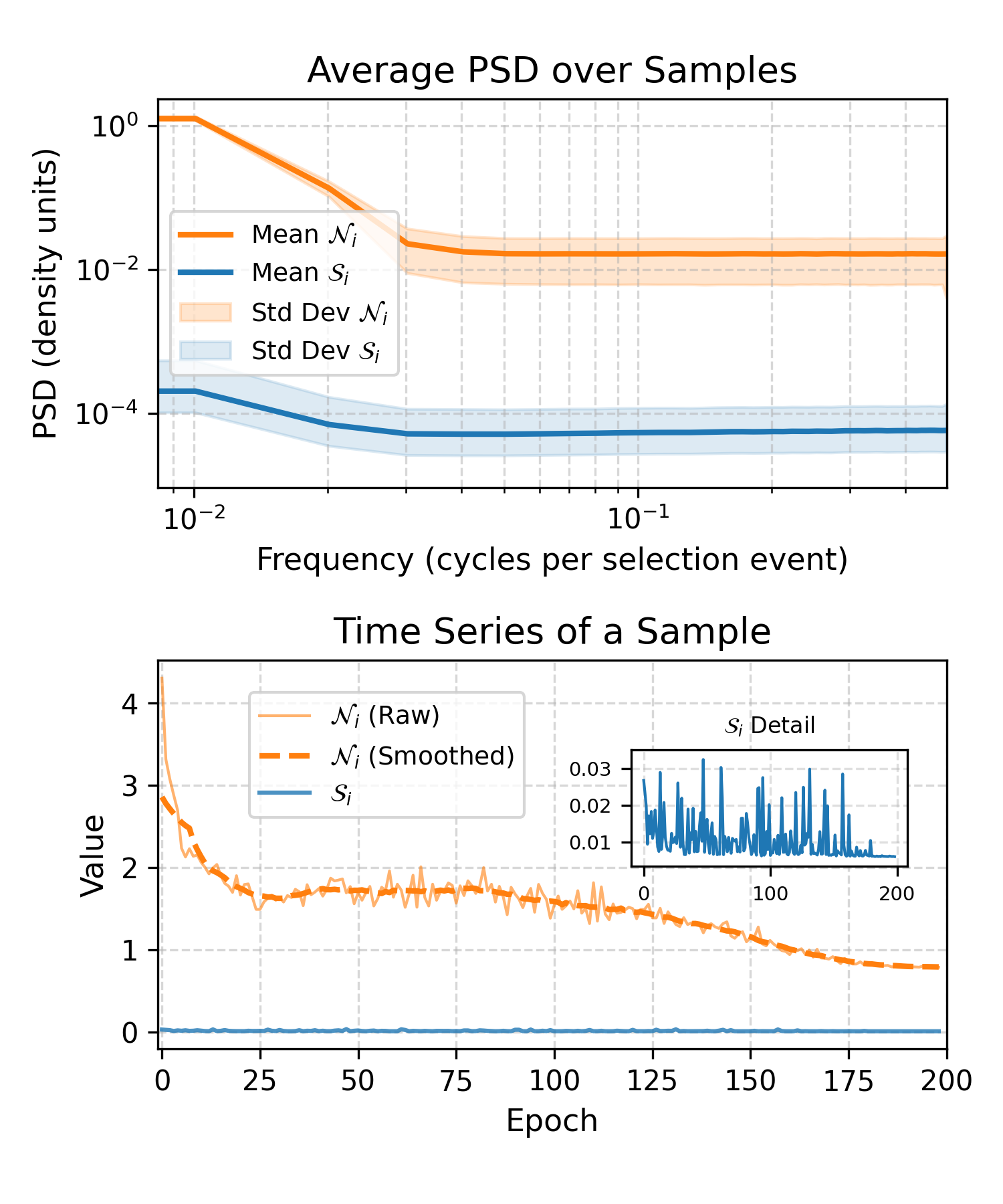} 
    \vspace{-20pt}
    \caption{Empirical validation of frequency separation.
    \textbf{Top:} Average PSD for scaled signal ($\text{Mean } \mathcal{S}_i$) and batch noise ($\text{Mean } \mathcal{N}_i$).
    \textbf{Bottom:} Time series for a representative sample.
    $\mathcal{S}_i$ magnitude is much smaller (see inset).
    }
    \label{fig:freq_separation_combined} 
    \vspace{-\intextsep} 
\end{figure}

To empirically validate our Frequency Separation Hypothesis (\cref{ass:frequency_final}), we analyzed the spectral characteristics of the scaled signal $\mathcal{S}_i$ and batch composition noise $\mathcal{N}_i$ components derived from the mean batch loss decomposition (Eq.~\ref{eq:loss_decomp_final}). We computed the Power Spectral Density (PSD) using Welch's method \citep{Welch1967Use} for sequences collected during standard training (ResNet18 on CIFAR100).

\Cref{fig:freq_separation_combined} presents the average PSD over 50K samples and 200 epochs and a representative time series. The results provide strong empirical support for our hypothesis. The average PSD (\cref{fig:freq_separation_combined}, top) demonstrates a clear separation: the noise component $\mathcal{N}_i$ (orange) possesses significantly higher energy density than the scaled signal $\mathcal{S}_i$ (blue) across all frequencies, with the disparity being particularly pronounced at higher frequencies. While $\mathcal{N}_i$ retains substantial high-frequency content, characteristic of stochastic batch sampling, $\mathcal{S}_i$'s energy is concentrated at lower frequencies, reflecting slower model dynamics. The time series plot (\cref{fig:freq_separation_combined}, bottom) visually confirms this, showing the high volatility of $\mathcal{N}_i$ around its trend compared to the small magnitude and smoother evolution of $\mathcal{S}_i$.

This observed frequency separation empirically justifies the core mechanism of BLS. The inherent low-pass filtering characteristic of the EMA update (\cref{prop:filtered_estimate_final}) effectively attenuates the dominant, high-frequency noise $\mathcal{N}_i$, enabling the BLS score $\score_i(t)$ to capture the persistent, low-frequency signal trend related to $\mathcal{S}_i$. This validates using BLS, derived solely from batch loss observations, as a efficient proxy for identifying samples with consistently high loss contributions.

\subsection{Ablation and Hyperparameter Analysis}
\label{sec:ablation_study}

\begin{table}[ht] 
    \footnotesize
    \centering
    \caption{Ablation study of BLS components on CIFAR100 (ResNet18/50). "P \%" denotes Pruning Percentage. "Time" includes pruning cost and training time. "+ BLS" refers to the full BLS integrated method. "+ w/o EMA BLS" uses only the last batch loss for scoring active samples.}
    \label{tab:ablation_bls}
    \setlength{\tabcolsep}{4pt} 
    \begin{tabular}{l|ccc|ccc} 
        \toprule[0.9pt]
        \multirow{2}{*}{Method} & \multicolumn{3}{c|}{R18} & \multicolumn{3}{c}{R50} \\
        & Acc & Time & P \% & Acc & Time & P \% \\
        \midrule[0.9pt]
        Full & 78.2 & 2.46h & - & 80.6 & 4.10h & - \\
        \hline
        InfoBatch  & 78.2 & 1.63h & 33.8 & 80.6 & 2.74h & 33.0 \\
        + w/o EMA BLS & 77.9 & 1.65h & 32.8 & 79.9 & 2.75h & 32.8 \\
        + BLS & \textbf{78.3} & 1.63h & 33.9 & \textbf{80.6} & 2.74h & 33.1 \\
        \hline
        SeTa  & 78.3 & 1.53h & 40.0 & 80.7 & 2.53h & 39.8 \\
        + w/o EMA BLS  & 78.0 & 1.53h & 39.8 & 79.9 & 2.55h & 38.8 \\
        + BLS  & \textbf{78.4} & 1.52h & 40.1 & \textbf{80.8} & 2.52h & 40.1 \\
        \bottomrule[0.9pt]
    \end{tabular}
\end{table}

\textbf{Ablation Analysis.}
We ablate BLS components on CIFAR100 (ResNet18/50), with results in Table~\ref{tab:ablation_bls}. We compare against full training, original per-sample loss methods (InfoBatch, SeTa), BLS using only the most recent batch loss for active samples ("+ w/o EMA BLS"), and full BLS integration ("+ BLS").
The data underscores the criticality of the EMA for temporal smoothing. Removing EMA ("+ w/o EMA BLS") consistently degrades accuracy compared to full BLS across both InfoBatch and SeTa integrations (e.g., R18 InfoBatch: 77.9\% vs. 78.3\%; R50 SeTa: 79.9\% vs. 80.8\%). This highlights that instantaneous batch losses alone are insufficient, and EMA's noise filtering is essential for robust scoring.
Furthermore, full BLS-integrated methods ("+ BLS") achieve accuracies on par with, or slightly exceeding, their original per-sample loss counterparts, while maintaining similar data pruning percentages (P \%). For instance, BLS-SeTa on R50 matches SeTa's accuracy (80.7\% vs. 80.8\%) with comparable pruning ($\sim$40\%).
 This validates BLS's complete mechanism.
Notably, the substantial data pruning achieved (e.g., 30-40\%) directly translates to reduced training time (Time column), establishing Pruned \% as a reliable proxy for computational savings, independent of hardware-specific timing variations.

\begin{figure}[ht]
    \vspace{-10pt}
    \centering
    \includegraphics[width=0.48\textwidth]{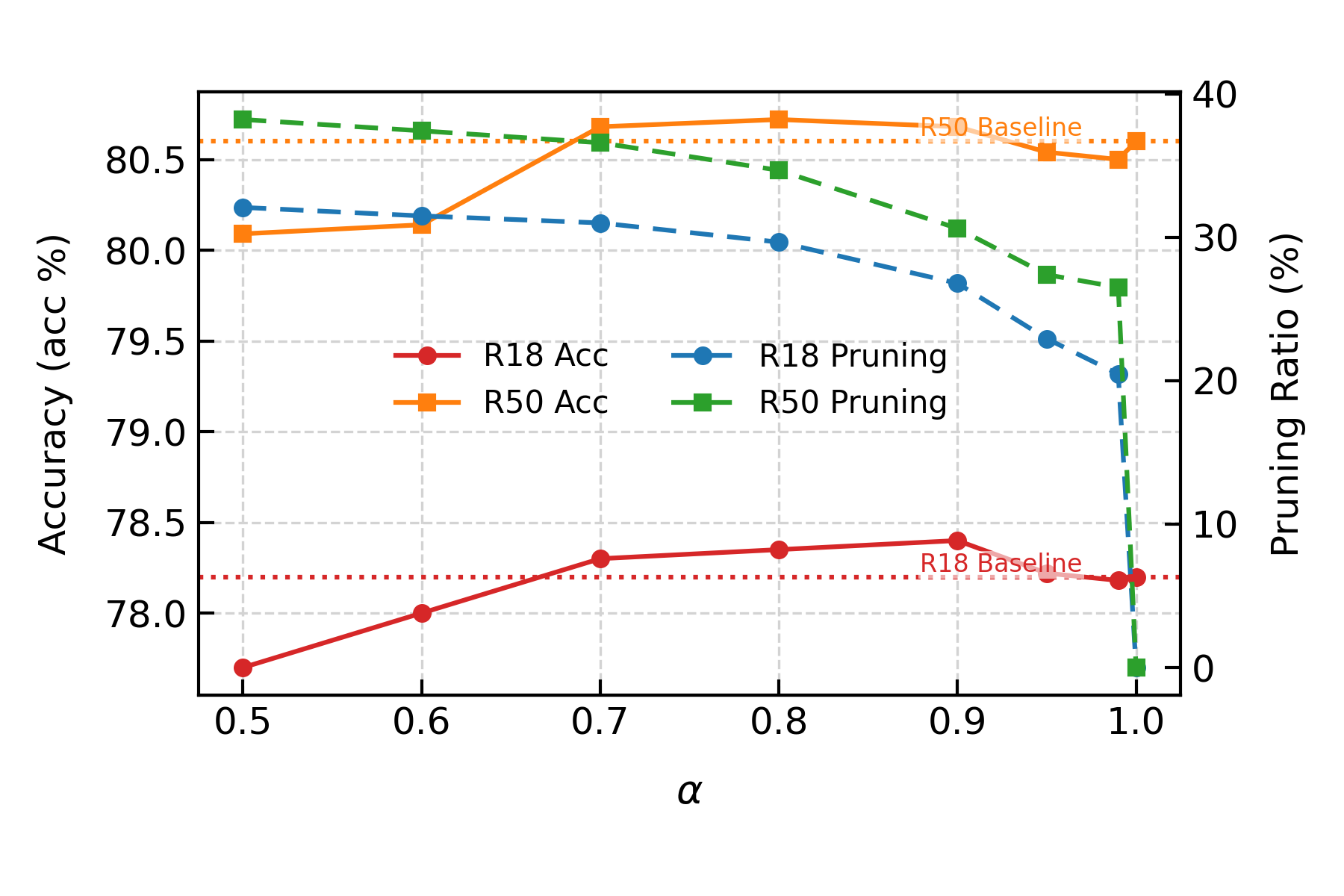} 
    \vspace{-25pt}
    \caption{Effect of EMA decay factor $\alpha$. Accuracy and Pruning Ratio are shown for ResNet18 and ResNet50.
    }
    \label{fig:alpha_sensitivity}
\end{figure}

\textbf{Hyperparameter Analysis.}
Figure~\ref{fig:alpha_sensitivity} illustrates BLS's sensitivity to the EMA decay factor $\alpha$ on CIFAR100 (ResNet18/50).
For $\alpha \in [0.5, 0.6]$, scores are highly responsive, yielding substantial pruning ratios (R18: $\sim$31-33\%, R50: $\sim$35-37\%) but accuracies (R18: $\sim$79.0-79.1\%, R50: $\sim$80.1-80.2\%) are slightly below their respective baselines, likely due to noise sensitivity.
As $\alpha$ increases to the range $[0.7, 0.8]$, a favorable trade-off is observed. ResNet18 accuracy peaks at 79.4\% (e.g., $\alpha=0.7, 0.8$) with a pruning ratio of $\sim$30-31\%. ResNet50 accuracy also peaks at 80.7\% ($\alpha=0.8$) with $\sim$30\% pruning, surpassing its baseline. This region suggests effective noise filtering while maintaining responsiveness.
For $\alpha \in [0.9, 1.0)$, pruning ratios decline sharply for both models as scores become overly smoothed and static, with $\alpha=1.0$ effectively replicating full training (near-zero pruning and baseline accuracy).
This analysis indicates that $\alpha$ values between 0.7 and 0.8 offer a robust balance of high accuracy and significant pruning. Consequently, we default to $\alpha=0.7$ for other experiments.

\section{Conclusion}
\label{sec:conclusion}

This paper introduces the BLS, a practical and efficient method for dynamic data pruning that uniquely infers sample importance solely from readily available mean batch losses, bypassing the complexities of accessing per-sample loss. By applying a sample-specific Exponential Moving Average to the sequence of batch losses encountered by each sample, BLS effectively filters noise arising from stochastic batch composition, as rigorously demonstrated by our signal-processing analysis. Extensive empirical validation across ten diverse tasks, architectures, and learning paradigms confirms that BLS enables significant data pruning while maintaining performance comparable to state-of-the-art methods reliant on direct per-sample losses. Its negligible overhead and implementation simplicity make BLS a broadly applicable tool for enhancing various training tasks.

\subsubsection*{Acknowledgments}
This work was supported by the National Natural Science Foundation of China under Grant 62471394, U21B2041, 62306241 and 62576284.

\bibliographystyle{abbrv}
{
	\small
	\bibliography{ref}
}

\clearpage
\setcounter{page}{1}
\maketitlesupplementary

\begin{algorithm*}[htbp]
    \small
    \caption{BLS Integration: \textbf{\textcolor{red}{One-Line Proxy}} and \textbf{\textcolor{blue}{Three-Line Injection}}}
    \label{alg:bls_integration}
    \begin{algorithmic}[1]
        \State \textbf{Input:} Original dataset $\mathcal{D}_{\text{train}}$, framework args $\theta_{\text{Framework}}$, BLS decay $\alpha$.
        \State \textbf{Output:} Trained model parameters $\mathbf{W}^*$.
        \Statex
        \State $\text{InfoBatchInst} \gets \text{InfoBatch}(\mathcal{D}_{\text{train}}, \theta_{\text{Framework}})$ \Comment{\textcolor{blue}{\#1 Wrap dataset}}
        \State $\boxed{\text{DataHandler} \gets \text{BLS}(\text{InfoBatchInst}, \alpha).\text{proxy}()}$ \Comment{\textcolor{red}{One-Line Proxy: BLS modifies InfoBatch's update}}
        \State Loader $\gets \text{DataLoader}(\text{DataHandler}, \text{sampler}=\text{DataHandler}.\text{sampler}, \dots)$ \Comment{\textcolor{blue}{\#2 Inject sampler}}
        \Statex
        \State Initialize model $\mathcal{M}(\mathbf{W})$, optimizer $\mathcal{O}$
        \State Criterion $\mathcal{C} \gets \text{LossFunction}(\text{reduction="mean"})$ \Comment{BLS uses standard batch loss}
        \For{epoch $\in [1, \dots, N_{\text{epochs}}]$}
            \For{($\bx_{\text{batch}}$, $\by_{\text{batch}}$) \textbf{in} Loader} \Comment{\textit{Indices are set on DataHandler internally by DataLoader hack}}
                \State $\mathcal{O}.\text{zero\_grad}()$
                \State $\batchloss_t \gets \mathcal{C}(\mathcal{M}(\bx_{\text{batch}}), \by_{\text{batch}})$ \Comment{Obtain standard mean batch loss}
                \State $\text{loss}_{\text{final}} \gets \text{DataHandler}.\text{update}(\batchloss_t)$  \Comment{\textcolor{blue}{\#3 Proxied update uses $\batchloss_t$}}
                \Statex \hspace{\algorithmicindent} \Comment{\textit{Internally, performs BLS EMA update on current batch samples}}
                \State $\text{loss}_{\text{final}}.\text{backward}()$
                \State $\mathcal{O}.\text{step}()$
            \EndFor
        \EndFor
        \State \textbf{return} Trained model parameters $\mathbf{W}^*$
    \end{algorithmic}
\end{algorithm*}
\section{Simplicity of BLS: \textit{One-Line Proxy} and Seamless Integration}
\label{sec:bls_simplicity_integration}

A hallmark of BLS is its exceptional ease of integration into existing training workflows, particularly when enhancing dynamic data selection frameworks like InfoBatch \citep{InfoBatch} or SeTa \citep{seta} that were originally designed around per-sample losses ($\loss_i(t)$). BLS achieves this through a conceptual \textbf{\textit{one-line proxy}} mechanism that seamlessly modifies the score generation within such frameworks, coupled with an intrinsically simple update logic. This entirely obviates the need for direct $\loss_i(t)$ computation, a significant practical advantage as shown in Figure~\ref{fig:viz2_flow} and detailed in \cref{appendix:difficulty_per_sample_loss}.

Algorithm~\ref{alg:bls_integration} illustrates this streamlined integration, contextualized with a framework like InfoBatch. The core idea is that BLS wraps the existing data handling object (e.g., an \texttt{InfoBatchInstance}) and, through its conceptual \texttt{.proxy()} method (realized by the BLS class constructor in our implementation), effectively overrides or augments its internal score update mechanism to use the mean batch loss instead of per-sample losses.

The \textbf{\textit{one-line proxy}} is demonstrated at Line~2 of Algorithm~\ref{alg:bls_integration}:
\texttt{DataHandler $\gets$ BLS(InfoBatchInst, $\alpha$).proxy()}.
Here, \texttt{InfoBatchInst} is an initialized instance of a base pruning framework. The \texttt{BLS} wrapper effectively modifies this instance (now referred to as \texttt{DataHandler}) such that its subsequent \texttt{.update()} calls will utilize BLS's scoring logic, which operates on the mean batch loss. This single line conceptually redirects the core of the importance scoring from a per-sample loss basis to a batch loss basis.

The \textbf{\textit{three-line injection}} then highlights the minimal conceptual changes to a standard training loop when using this BLS-enhanced data handler:
\begin{enumerate}
    \itemsep0em
    \item \textbf{Line 3 (Framework Wrapping):} \texttt{InfoBatchInst $\gets$ InfoBatch(...)}. This is the standard initialization of the base pruning framework (e.g., InfoBatch). The subsequent BLS proxy (Line 4) builds upon this.
    \item \textbf{Line 5 (Sampler Integration):} \texttt{Loader $\gets$ DataLoader(DataHandler, sampler=DataHandler.sampler)}. The DataLoader utilizes the custom sampler provided by the \texttt{DataHandler} (the BLS-proxied InfoBatch instance). This sampler now implicitly operates based on scores that will be generated and maintained by BLS.
    \item \textbf{Line 12 (Proxied Update Call):} \texttt{loss\_final $\gets$ DataHandler.update($\batchloss_t$)}. This line, within the training loop, calls the \texttt{update} method. Crucially, due to the BLS proxy, this method now takes the standard mean batch loss $\batchloss_t$ (obtained on Line 11 using a criterion with \texttt{reduction="mean"}) as its primary input for scoring. As noted in the algorithm's comments, the internal logic of this proxied \texttt{update} method then executes BLS's core EMA calculation for samples in the current batch using this $\batchloss_t$. The returned \texttt{loss\_final} might be the original $\batchloss_t$ or a version rescaled by the underlying framework's (e.g., InfoBatch's) mechanism for gradient adjustment.
\end{enumerate}
Thus, BLS leverages the existing infrastructure of methods like InfoBatch/SeTa for sampling and potential gradient adjustments but critically simplifies the score generation by using the universally available mean batch loss. This minimal footprint and reliance only on $\batchloss_t$ underscore BLS's practical utility and ease of adoption for efficient dynamic data pruning.

\section{On the Practical Difficulty of Obtaining Per-Sample Losses}
\label{appendix:difficulty_per_sample_loss}

The main text highlights the practical challenges in obtaining per-sample losses $\loss_i(t)$ for dynamic data selection. This appendix provides a more granular discussion of these difficulties, underscoring the motivation for methods like BLS that operate solely on aggregated batch losses. We first discuss general framework-level hurdles and then delve into task-specific complexities using object detection as a prime example, finally illustrating the differing access complexities with pseudocode.

\begin{figure}
    \centering
    \includegraphics[width=0.5\textwidth]{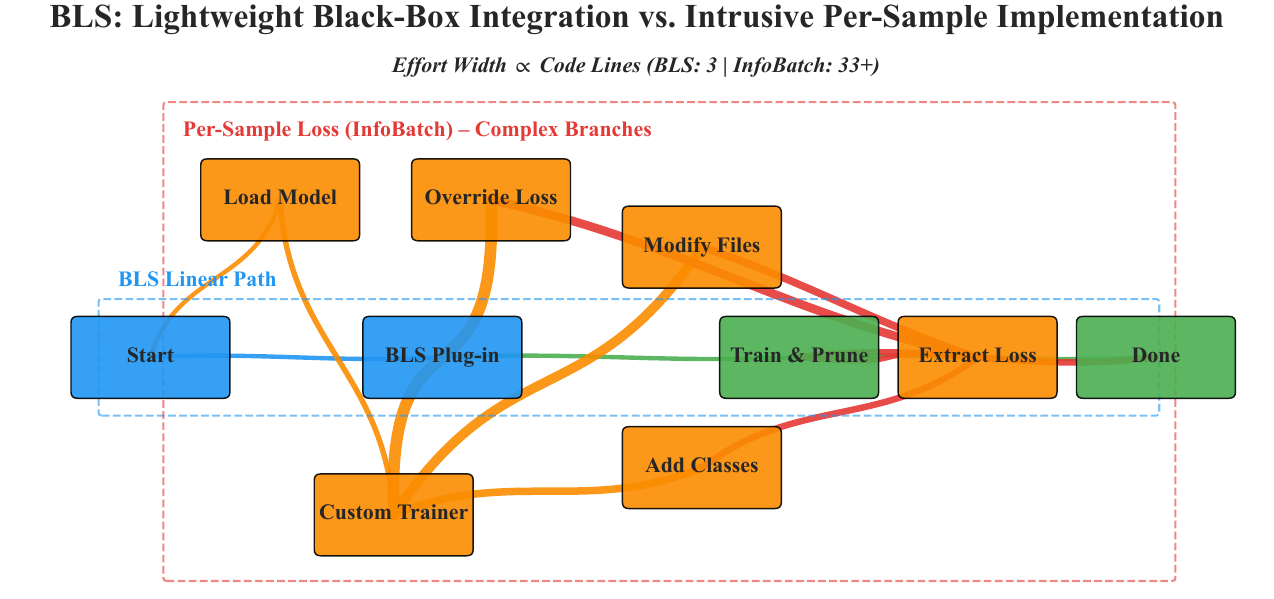}
    \caption{Workflow Comparison: BLS Black-Box Simplicity vs. Per-Sample Loss Complexity.}
    \label{fig:viz2_flow}
\end{figure}

\subsection{Standard Loss Aggregation and Framework Encapsulation}
\label{app:sub:framework_aggregation}

Modern deep learning frameworks, such as PyTorch \citep{paszke2019pytorch}, are architected primarily for computational efficiency. A key aspect of this design is the default aggregation of losses at the batch level. Standard loss function implementations typically perform an immediate reduction (e.g., `reduction='mean'` or `reduction='sum'`), yielding a single scalar batch loss which serves as the primary input for backpropagation. Accessing individual sample losses $\loss_i(t)$ therefore necessitates a deviation from this default, requiring explicit configuration of the loss function (e.g., `reduction='none'`) to return a tensor of per-sample losses (e.g., shape $[B]$ for batch size $B$). While technically feasible for certain loss types, this explicit non-reduction forces modifications to standard training pipelines, as training loops and gradient computation logic are typically designed for scalar losses. Furthermore, high-level training libraries and APIs (e.g., YOLO \citep{yolo}) often introduce further layers of abstraction over the loss computation. Unpacking these abstractions to expose per-sample losses can require intrusive modifications to internal library code, thereby increasing implementation complexity and reducing the portability of such custom solutions.

\subsection{Case Study: Per-Sample Loss in Object Detection}
\label{app:sub:object_detection_loss}

Object detection serves as a salient example where defining and extracting a representative per-sample loss $\loss_i(t)$ is highly non-trivial. For a single input image $i$ in a typical detector (e.g., YOLO \citep{yolo}), the loss computation involves several stages. Initially, numerous region proposals or anchor boxes are generated and matched against ground truth objects. For each matched positive proposal $p$ associated with image $i$, at least two loss components are computed: a classification loss $\loss_{\text{cls}}(p)$ and a bounding box regression loss $\loss_{\text{reg}}(p)$. Background proposals typically only incur a classification loss. The total loss for image $i$, denoted $\loss_i^{\text{task}}$, is then an aggregation, often a weighted sum, of these individual proposal/anchor losses:

\begin{equation}
\begin{split} 
    \loss_i^{\text{task}} = & \sum_{p \in \text{pos\_proposals}_i} (\lambda_1 \loss_{\text{cls}}(p) + \lambda_2 \loss_{\text{reg}}(p)) + \\ 
    & \sum_{p \in \text{neg\_proposals}_i} \lambda_3 \loss_{\text{cls}}(p)_{\text{bg}},
\end{split}
\label{eq:obj_det_loss_sample}
\end{equation}

where $\lambda$ are weighting factors. This $\loss_i^{\text{task}}$ represents the per-sample loss one would ideally want. Finally, the batch loss $\batchloss$ used for gradient updates is the average of these $\loss_i^{\text{task}}$ values over all images $i$ in the batch $\batch_t$.

The primary challenge in obtaining $\loss_i^{\text{task}}$ lies in the encapsulation by object detection frameworks. These frameworks typically integrate the entire multi-stage loss calculation, outputting only the final scalar $\batchloss(\batch_t, t)$. Extracting the intermediate $\loss_i^{\text{task}}$ for each image usually requires modifying these complex, highly optimized internal modules, as a simple `reduction='none'` at the image level is often unavailable or insufficient. Even if $\loss_i^{\text{task}}$ could be extracted, its utility as a single "hardness" indicator is debatable. It aggregates potentially hundreds of sub-losses; an image with one severe localization error might yield a similar $\loss_i^{\text{task}}$ to an image with many minor classification errors, despite representing different types of model failure. Furthermore, the variable number of proposals contributing to $\loss_i^{\text{task}}$ across images complicates direct comparisons without careful normalization. Thus, for complex tasks like object detection, obtaining a meaningful and accessible $\loss_i(t)$ often necessitates deep architectural and task-specific interventions.

\begin{algorithm*}[htbp]
\caption{Illustrative Comparison of Loss Access (Object Detection as an Example)}
\label{alg:loss_access_detailed}
\begin{algorithmic}[1]
\State \textbf{Input:} Model $f(\cdot; \theta)$ (Backbone, RPN, Head), Batch $\batch_t = \{(\mathbf{x}_k, \mathbf{y}_k)\}_{k=1}^B$
\Statex
\Procedure{GetBatchLossStandard}{$\batch_t, f$} \Comment{Standard, Direct Path}
    \State $\text{Features}_{\text{batch}} \leftarrow f_{\text{backbone}}(\mathbf{X}_{\text{batch}})$
    \State $\text{Proposals}_{\text{batch}}, \loss_{\text{RPN}}^{\text{batch}} \leftarrow f_{\text{RPN}}(\text{Features}_{\text{batch}}, \mathbf{Y}_{\text{batch}})$
    \State $\text{Detections}_{\text{batch}}, \loss_{\text{Head}}^{\text{batch}} \leftarrow f_{\text{head}}(\text{Features}_{\text{batch}}, \text{Proposals}_{\text{batch}}, \mathbf{Y}_{\text{batch}})$
    \State $\batchloss \leftarrow (\loss_{\text{RPN}}^{\text{batch}} + \loss_{\text{Head}}^{\text{batch}}) / B$ \Comment{Framework aggregates internally}
    \State \textbf{return} $\batchloss$
\EndProcedure

\Statex
\Procedure{GetPerSampleLossesComplex}{$\mathcal{B}_t, f$} \Comment{Involved Path (e.g., Object Detection)}
    \State $\text{per\_sample\_losses}[1..B] \leftarrow \text{InitializeArray}(B)$
    \For{$k = 1 \textbf{ to } B$} \Comment{Iterate through each image in the batch}
        \State $(\mathbf{x}_k, \mathbf{y}_k) \leftarrow \mathcal{B}_t[k]$
        \State $\text{Features}_k \leftarrow f_{\text{backbone}}(\mathbf{x}_k)$
        \State $\text{Proposals}_k, \loss_{\text{RPN}}(k) \leftarrow f_{\text{RPN}}(\text{Features}_k, \mathbf{y}_k)$ \Comment{Loss for RPN on image $k$}
        \State $\text{SampledProposals}_k \leftarrow \text{FilterAndSampleProposals}(\text{Proposals}_k, \mathbf{y}_k)$
        \State $\loss_{\text{Head}}(k) \leftarrow 0$
        \For{each proposal $p$ in $\text{SampledProposals}_k$}
            \State $\text{RoIFeatures}_p \leftarrow \text{RoIAlign}(\text{Features}_k, p.\text{coords})$
            \State $\text{class\_logits}_p, \text{box\_deltas}_p \leftarrow f_{\text{head\_branches}}(\text{RoIFeatures}_p)$
            \State $\loss_{\text{cls}}(p), \loss_{\text{reg}}(p) \leftarrow \text{ComputeProposalLosses}(\text{class\_logits}_p, \text{box\_deltas}_p, p, \mathbf{y}_k)$
            \State $\loss_{\text{Head}}(k) \leftarrow \loss_{\text{Head}}(k) + \lambda_1 \loss_{\text{cls}}(p) + \lambda_2 \loss_{\text{reg}}(p)$ \Comment{Accumulate head losses}
        \EndFor
        \If{$|\text{SampledProposals}_k| > 0$}
            \State $\loss_{\text{Head}}(k) \leftarrow \loss_{\text{Head}}(k) / |\text{SampledProposals}_k|$ \Comment{Normalize head loss}
        \EndIf
        \State $\text{per\_sample\_losses}[k] \leftarrow \loss_{\text{RPN}}(k) + \loss_{\text{Head}}(k)$ \Comment{Total loss for image $k$}
    \EndFor
    \State \textbf{return} $\text{per\_sample\_losses}$
\EndProcedure
\end{algorithmic}
\end{algorithm*}

\subsection{Illustrative Pseudocode: Accessing Batch vs. Per-Sample Loss}
\label{app:sub:pseudocode}

The difference in accessibility can be further illustrated with simplified pseudocode. Algorithm~\ref{alg:loss_access_detailed} contrasts the typical workflow for obtaining batch loss versus the more involved process for per-sample losses, especially reflecting the complexities of tasks like object detection.

Algorithm~\ref{alg:loss_access_detailed} illustrates that `GetBatchLossStandard` relies on the framework's internal aggregation to produce a single scalar. In contrast, `GetPerSampleLossesComplex` (simulating object detection) requires explicitly iterating through each image, re-implementing or exposing the multi-stage loss calculation (RPN loss, proposal sampling, RoI feature extraction, head losses, and their weighted aggregation) to obtain an individual $\loss_k^{\text{task}}$. This significantly increases complexity compared to merely setting `reduction='none'` in a simpler loss function.

\subsection{Implementation Complexity and Maintainability}
\label{app:sub:implementation}
Beyond direct computational costs or conceptual challenges in defining $\loss_i(t)$, the integration of per-sample loss extraction and its subsequent utilization invariably entails increased engineering effort and long-term maintenance burdens. Altering standard training scripts, which are typically well-tested and optimized for scalar loss workflows, introduces potential for errors and diverges from established community practices. More critically, the logic to correctly define, isolate, or appropriately weigh components to form a meaningful $\loss_i(t)$ is inherently task-dependent and often architecture-specific, as exemplified by object detection (\cref{app:sub:object_detection_loss}) and also noted for other complex scenarios such as certain Semi-Supervised Learning \citep{usb} and YOLO \citep{yolo} frameworks. This necessity for bespoke code for each new application or significant model change curtails generalizability and reduces the portability of such per-sample loss-based techniques across different research codebases or evolving framework versions. Consequently, custom logic for per-sample loss handling adds a considerable maintenance overhead, as changes in underlying libraries or model architectures may necessitate revisiting and adapting this specialized code.

In summary, the path to obtaining and utilizing per-sample losses $\loss_i(t)$ is often laden with practical challenges. These range from navigating framework abstractions and defining meaningful scalars for complex tasks like object detection, to the intricacies of disentangling contributions in interactive loss settings, and the general increase in implementation complexity and maintenance. These cumulative difficulties provide strong motivation for developing alternative sample importance measures, like BLS, that can effectively leverage the universally accessible mean batch loss.

\section{Efficiency Metric Justification.}
\label{appendix:efficiency_details}

\begin{table}[ht]
    \footnotesize
    \setlength{\tabcolsep}{4pt}
    \caption{Computational overhead of BLS with InfoBatch and SeTa (1M samples, NVIDIA RTX 3090 GPU).
    }
    \centering
    \label{tab:efficiency_bls}
    \begin{tabular}{c|ccc}
        \toprule[0.9pt]
        Method & Overhead & \makecell{R18 P/B \\ 734.1s} & \makecell{R50 P/B \\ 2122.4s} \\
        \midrule[0.9pt]
        InfoBatch & 0.236s & 0.03\% & 0.01\% \\
        BLS-InfoBatch & 0.253s & 0.03\% & 0.01\% \\
        \hline
        SeTa & 10.001s & 1.4\% & 0.5\% \\
        BLS-SeTa & 10.021s & 1.4\% & 0.5\% \\
        \bottomrule[0.9pt]
    \end{tabular}
\end{table}

Reproducing wall-clock time reductions for dynamic pruning is notoriously challenging due to hardware and system dependencies. We therefore primarily report the \textbf{percentage of data pruned (Pruned \%)} as our efficiency metric. This choice is validated by the minimal computational overhead. \Cref{tab:efficiency_bls} demonstrates that integrating BLS into existing pruning methods (InfoBatch, SeTa) introduces negligible additional overhead ($<0.02$s for 1M samples) compared to the original methods. More importantly, the total overhead of these BLS-enhanced pruning methods remains extremely low relative to backbone processing time for representative models like ResNet18 and ResNet50.
This minimal overhead ensures compute positivity \citep{DBLP:conf/eccv/EvansPMSTH24,seta}, where the cost of pruning is vastly outweighed by the savings from data reduction (at least $>20$\% in our experiments). Furthermore, for complex models involving multiple components beyond a single backbone (e.g., ViECap with GPT2), the relative pruning overhead would be even smaller. Thus, Pruned \% serves as a robust, hardware-agnostic, and easily comparable indicator of the potential computational savings offered by BLS-guided pruning.

\section{Limitations}
\label{sec:limitations}
While BLS offers significant advantages in simplicity and applicability, its current operational scope has some considerations. Primarily, BLS functions as a scoring mechanism. When integrated into existing dynamic pruning frameworks like InfoBatch \citep{InfoBatch} or SeTa \citep{seta}, it replaces their per-sample loss-based scoring but does not inherently alter their multi-epoch iterative pruning schedules or curriculum designs. Thus, while BLS enables these frameworks to operate without direct per-sample loss access, the overall training process duration remains influenced by the epoch-dependent nature of the specific pruning strategy and BLS itself. 


\end{document}